\documentclass[letterpaper, 10 pt, conference]{ieeeconf}
\usepackage{amsmath,amssymb,euscript,yfonts,psfrag,latexsym,graphicx}
\usepackage{bbm,color,amstext,wasysym,parskip}

\graphicspath{{./},{../figures/},{./figures/}}
\usepackage{caption}
\usepackage{subcaption}
\usepackage{float}
\usepackage{array}
\usepackage{url}
\usepackage{algorithm}
\usepackage{algorithmic}
\usepackage{xcolor}

\usepackage[normalem]{ulem}

\usepackage{hyperref}

\graphicspath{{./},{../figures/},{./figures/}}

\newtheorem{thm}{Theorem}

\newtheorem{rem}{Remark}

\usepackage{soul}

\newcommand{\bo}{{\mathbf o}}

\newcommand{\cN}{{\mathcal N}}

\newcommand{\bx}{{\bf x}}

\def\minwrt[#1]{\underset{#1}{\text{minimize}}}
\def\maxwrt[#1]{\underset{#1}{\text{maximize }}}

\IEEEoverridecommandlockouts      
\begin{document}
\title{\LARGE \bf Inference of collective Gaussian hidden Markov models}

% \author{ABC
% \thanks{This work was supported by .
% }% <-this % stops a space
% \thanks{Authors are with the School of Aerospace Engineering,
% Georgia Institute of Technology, Atlanta, GA, USA.}}

\author{Rahul Singh and Yongxin Chen
\thanks{This work was supported by NSF under grant 1942523 and 2008513.}
\thanks{R. Singh and Y.\ Chen are with the School of Aerospace Engineering,
Georgia Institute of Technology, Atlanta, GA; {\{rasingh@gatech.edu,yongchen\}@gatech.edu}}}

\maketitle
\begin{abstract}
% We consider inference problem in continuous state collective hidden Markov models (HMMs), where the data is recorded in aggregate (collective) form generated by a large population of individuals following the same HMM. In particular, we propose an aggregate inference algorithm called collective Gaussian forward-backward algorithm, extending recently proposed Sinkhorn belief propagation algorithm to HMMs characterized by Gaussian densities. We also show that our algorithm reduces to the standard Kalman filter when the observations are generated by a single individual. Our algorithm enjoys convergence guarantee and its efficacy is demonstrated through multiple experiments. 
We consider inference problems for a class of continuous state collective hidden Markov models, where the data is recorded in aggregate (collective) form generated by a large population of individuals following the same dynamics. We propose an aggregate inference algorithm called collective Gaussian forward-backward algorithm, extending recently proposed Sinkhorn belief propagation algorithm to models characterized by Gaussian densities. Our algorithm enjoys convergence guarantee. In addition, it reduces to the standard Kalman filter when the observations are generated by a single individual. The efficacy of the proposed algorithm is demonstrated through multiple experiments. 
\end{abstract}

%%%%%%%%%%%%%%%%%%%%%%%%%%%%%%%%%%%%%%%%%%%
\section{Introduction}
\label{sec:intro}
Filtering and smoothing problems in dynamic systems have been studied in a variety of applications including navigation, object tracking, robot localization, and control~\cite{LorNae11,And79}. These problems can be more generally posed as a Bayesian inference problem in probabilistic graphical models (PGMs)~\cite{WaiJor08}. For instance, the Kalman filter model is a special type of PGMs known as hidden Markov model (HMM) with linear Gaussian state and observation densities, which we call Gaussian HMM in this paper. The traditional filtering and smoothing methods are based on the observations recorded by a single individual.

%HMMs are widely used statistical models of time varying observation sequences that are indicative of the underlying true state sequence of a dynamic system.

The problem of modeling population-level observations has gained much attention in recent years~\cite{SinHaaZha20,SheDie11}. The observations are often recorded in aggregate form in terms of population counts due to various reasons such as privacy concerns, cost of data collection, and measurement fidelity. A timely example is modeling the spread of COVID-19 in a geographical area over time. In this example, the observation sequence may represent aggregate features including number of deaths and number of positives, based on which one might be interested in predicting number of infections in future. Collective filtering have emerged as a method to model such aggregate observations. Some more example applications of collective filtering include bird migration analysis~\cite{SinHaaZha20_incremental} and estimating the crowd flow in an urban environment~\cite{HasRinChe19}. Traditional filtering algorithms such as belief propagation~\cite{Pea88} and the Kalman filter are not applicable to collective settings due to data aggregation.

Collective filtering has been studied under the umbrella of recently proposed more general framework known as collective graphical models (CGMs)~\cite{SheDie11,SinHaaZha20}. Various aggregate inference algorithms under the CGM framework has been proposed including approximate MAP inference~\cite{SheSunKumDie13}, non-linear belief propagation~\cite{SunSheKum15}, and Sinkhorn belief propagation (SBP)~\cite{SinHaaZha20}. The collective forward-backward algorithm (CFB)~\cite{SinHaaZha20} is a special case of SBP for aggregate inference in CGMs. However, all of these algorithms are only applicable to CGMs with discrete states. A very recent work on filtering for continuous time HMMs is collective feedback particle filter~\cite{KimMeh20}.

In this paper, we are interested in inference problem in collective HMMs with continuous states. The continuous observations recorded by a large number of agents, with each agent following a certain underlying model, are aggregated such that the individuals are indistinguishable from the measurements. In such data aggregation, the individual's association with the measurements remains unknown. We propose an algorithm for estimating aggregate state distributions in collective HMMs with Gaussian densities, which we call the collective Gaussian forward-backward (CGFB) algorithm. Our algorithm is based on the recently proposed CFB algorithm~\cite{SinHaaZha20,SinHaaZha20_incremental}, which has convergence guarantees. We also propose a sliding window filter for faster inference extending SW-SBP~\cite{SinHaaZha20_incremental} to Gaussian HMMs. Furthermore, we show that in case of a single individual, our algorithm naturally reduces to the standard Kalman filter and smoother. We demonstrate the performance of algorithm via multiple numerical experiments. 

Rest of the paper is organized as follows. In Section~\ref{sec:background}, we briefly discuss related background including collective HMMs. We present our main results and algorithm in Section~\ref{sec:main_result}. Next, we discuss the connections of our algorithm with the standard Kalman filter in Section~\ref{sec:connection_kalman}. This is followed by experimental results in Section~\ref{sec:experiments} and conclusion in Section~\ref{sec:conclusion}.

%%%%%%%%%%%%%%%%%%%%%%%%%%%%%%%%%
\section{Background}
\label{sec:background}
In this section, we briefly discuss related background including HMMs, collective HMMs, and collective forward-backward algorithm. 
%===================================================
\subsection{Hidden Markov Models}
\label{subsec:hmm}
Hidden Markov models (HMMs) consist of a Markov process describing the true (hidden) state of a dynamic system and an observation process corrupted by noise. Let us denote the state variables $X_1,X_2,\ldots$ and observation variables $O_1,O_2,\ldots$. An HMM is parameterized by the initial distribution $p(X_1)$, the state transition probabilities $p(X_{t+1} \mid X_t)$, and the observation probabilities $p(O_{t} \mid X_{t})$ for each time step $t = 1,2,\ldots$. An HMM of length $T$ is represented graphically as shown in Figure~\ref{fig:hmm_model}. The joint distribution of an HMM with length $T$ can be factorized as
\begin{equation} \label{eq:HMM_distribution}
    p(\bx, \bo) = p(x_1)~ \prod_{t=1}^{T-1} ~p(x_{t+1} \mid x_{t}) ~ \prod_{t=1}^{T}~p(o_t \mid x_t),
\end{equation}
where $\bx = \{x_1,x_2,\ldots, x_T\}$ and $\bo = \{o_1,o_2,\ldots, o_T\}$ represent a particular assignment of hidden and observation variables, respectively. %$p(x_1)$ is the initial distribution of the starting state, $p(x_{t+1} \mid x_{t})$ are the transition probabilities between hidden variables, and $p(o_{t} \mid x_{t})$ are the emission (observation) probabilities for time steps $t=1,2,\ldots, T$. 

The state and observation variables can take either discrete or continuous values. In this paper, we assume that both state as well as observation variables take continuous values unless otherwise specified. In standard inference problems over HMMs, the observations from a single individual $\hat{o}_1,\hat{o}_2,\ldots,\hat{o}_T$ are recorded over time, based on which the hidden state distributions $p(x_t|\hat{o}_{1:T})$ are estimated. The standard forward-backward algorithm~\cite{Rab89} is a popular method for inference over HMMs. Note that when the estimation of state distributions at time step $t$ is made based on only the current and past observations $\hat{o}_{1:t}$, it is called \textit{filtering} and when the future observations are also taken into account, it is more generally termed as \textit{inference} or \textit{smoothing}. 

%----------------------------
\begin{figure}[h]
\centering
\includegraphics[scale=0.8]{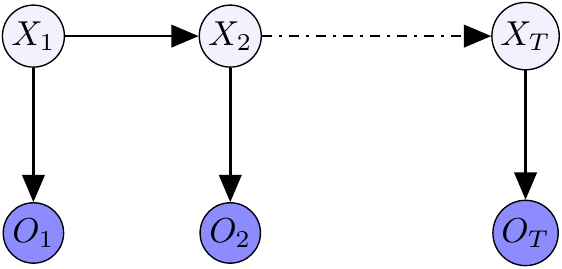}
\caption{Graphical representation of a length $T$ HMM.}
\label{fig:hmm_model}
\end{figure}
%------------------------------
%%%%%%%%%%%%%%%%%%%%%%%%%%%%%%%%%%%%%%%%%%%%%%%%%%%%%%%%%%%%%%%%%%%
\subsection{Collective Hidden Markov Models}
\label{subsec:collective_hmm}
Collective (aggregate) HMMs refer to the scenario when the data is generated by $M$ individuals independently transitioning from state to state according to the same Markov chain and the corresponding noisy observations are recorded in aggregate form such that the association to the corresponding state is not known. %Suppose we have $M$ number of individuals following a certain HMM. At each time step, observations from all the $M$ individuals are recorded in aggregate form such that individual observations are indistinguishable from each other.  
Let $X_t^{(m)}$ be the random variable representing the state of $m^{th}$ individual at time $t$ and $O_t^{(m)}$ be the observation variable of $m^{th}$ individual at time $t$. The observations are made in aggregate form $y_t(o_t)$ representing the distribution of the collective observations for each time step $t=1,2,\ldots,T$. The goal of inference in collective HMMs is to estimate the aggregate state distributions $n_t(x_t)$ given all the aggregate observation distributions. A pictorial representation of a collective HMM is depicted in Figure~\ref{fig:aggregate_model}. Here, $n_t(x_t)$ is an estimate of the state distribution of the $M$ agents at time step $t$.  

%----------------------------
\begin{figure}[h]
\centering
\includegraphics[scale=0.8]{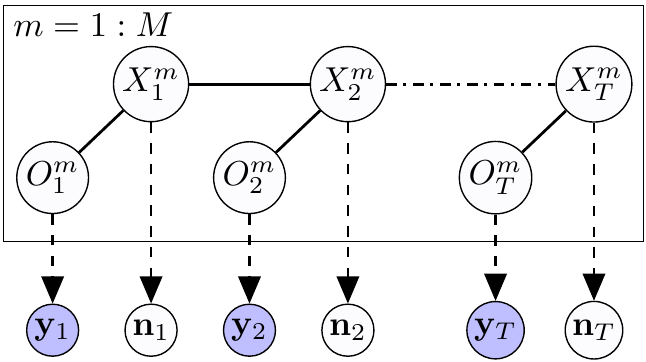}
\caption{Aggregate HMM (shaded nodes represent aggregate noisy observations).}
\label{fig:aggregate_model}
\end{figure}
%------------------------------

Inference in aggregate HMMs aims to estimate the aggregate hidden distributions based on indistinguishable aggregate measurements. Traditional inference algorithms such as forward-backward algorithms can not be used here due to data aggregation. The collective forward-backward (CFB) algorithm~\cite{SinHaaZha20} was recently proposed for aggregate inference in HMMs. It is a special case of more general aggregate inference algorithm known as Sinkhorn belief propagation (SBP)~\cite{SinHaaZha20}, which has convergence guarantee. %The CFB algorithm was originally proposed to handle aggregate inference problems in discrete state and discrete observation settings. 

%%%%%%%%%%%%%%%%%%%%%%%%%%%%%%%%%%%%%%%%%%%%%%%%%%%%%%%%%%%%%%%%%%%
\subsection{Collective Forward-Backward Algorithm} \label{subsec:collective_forward_backward}
The collective forward-backward algorithm~\cite{SinHaaZha20} is a message passing type inference method for aggregate HMMs. It employs propagating four different types of messages over the underlying HMM as shown in Figure~\ref{fig:hmm_message}. The CFB algorithm was originally proposed to handle discrete state and discrete observation HMMs and was recently extended to discrete state and continuous observation settings in \cite{ZhaSinChe20}. The messages in the CFB algorithm are computed as
%-----------------------------------------------------------
\begin{subequations}\label{eq:forward_backward_discrete}
\begin{eqnarray}
    \alpha_t(x_t) \propto \sum_{x_{t-1}}  p(x_t|x_{t-1}) \alpha_{t-1} (x_{t-1}) \gamma_{t-1}(x_{t-1}) , \label{eq:forward_backward1}  \\
    \beta_t(x_t) \propto  \sum_{x_{t+1}} p(x_{t+1}|x_{t}) \beta_{t+1} (x_{t+1}) \gamma_{t+1}(x_{t+1}), \label{eq:forward_backward2} \\
    \gamma_t(x_t) \propto \sum_{o_{t}} p(o_{t}|x_{t}) \frac{y_t(o_t)}{\xi_t(o_t)}, \label{eq:forward_backward3} \\
    \xi_t(o_t) \propto \sum_{x_{t}} p(o_{t}|x_{t}) \alpha_{t} (x_{t}) \beta_{t}(x_{t}), \label{eq:forward_backward4}
\end{eqnarray}
\end{subequations}
with boundary conditions 
\begin{equation*}
    \alpha_1(x_1) = p(x_1) \quad \text{and} \quad \beta_T(x_T) = 1.
\end{equation*}
%-----------------------------------------------------------

%----------------------------
\begin{figure}[h]
\centering
\includegraphics[scale=0.8]{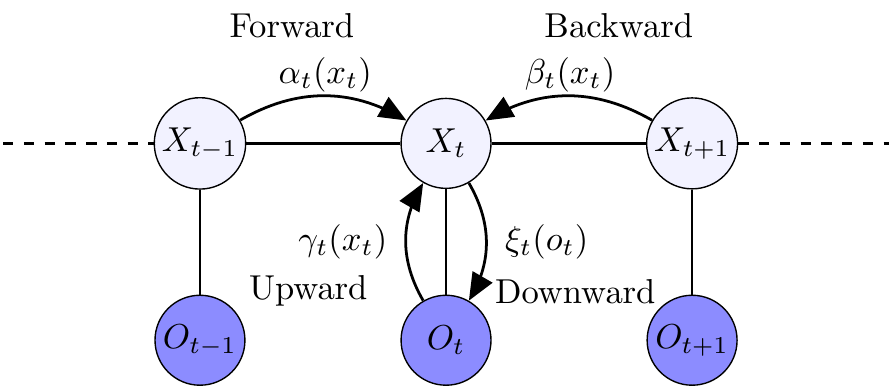}
\caption{Messages for inference in collective HMM.}
\label{fig:hmm_message}
\end{figure}
%------------------------------

The update sequence of these messages has a two-pass structure: a forward pass followed by a backward pass. In forward pass, $\gamma_t$, $\alpha_t$, and $\xi_t$ are updated $\forall t = 1,2,\ldots,T$, and in backward pass, $\gamma_t$, $\beta_t$, and $\xi_t$ are updated $\forall t = T, T-1,\ldots,1$. The algorithm is guaranteed to converge upon which, the required aggregate hidden state distributions are estimated as 
\begin{align*}
    n_t(x_t) \propto \alpha_t(x_t) \beta_t(x_t) \gamma_t(x_t),~\forall t= 1,2,\ldots, T.
\end{align*}

%%%%%%%%%%%%%%%%%%%%%%%%%%%%%%%%%%%%%
\section{Main Results}
\label{sec:main_result}

We propose collective Gaussian forward-backward (CGFB) algorithm for inference in aggregate Gaussian hidden Markov models (GHMMs).  
% %========================================================
% \subsection{Gaussian Hidden Markov Models}
% \label{subsec:hgmm}
GHMMs are special case of continuous state HMMs where the model densities take the form
%-----------------------------------------------------------
\begin{subequations}\label{eq:hmm_density}
\begin{eqnarray}
    p(X_{t+1}|X_{t},\theta_X) = \cN(x_{t+1}; A x_{t}, Q)  \label{eq:hgmm_density1}  \\
    p(O_t|X_t,\theta_1) = \cN(o_t; C x_t, R)  \label{eq:hgmm_density2}  \\
    p(X_1|\theta_1) = \cN(x_1; \pi, \Pi),  \label{eq:hgmm_density3} 
\end{eqnarray}
\end{subequations}

where $\theta_X = \{A,Q \}$, $\theta_O = \{C, R \}$, and $\theta_1 = \{\pi, \Pi \}$ are the parameters characterizing model densities in the HGMM model. Alternatively, the model densities for $t=1,2,\ldots$ can be characterized via the following set of equations.
\begin{subequations}
\begin{eqnarray}
   X_{t+1} =  A X_{t} + W_t \\
    O_t = C X_t + V_t, 
\end{eqnarray}
\end{subequations}
where $W_t \sim \cN(w_t; 0, Q)$, $V_t \sim \cN(v_t; 0, R)$, and $X_1 \sim \cN(x_1; \pi, \Pi)$.
% %================================================
% \subsection{Gaussian Refactorization Lemma}
% \label{subsec:GRL}

% Let $\bx$ be a random vector in $\mR^m$ and $\by$ in $\mR^n$. Let $\bA$ be a matrix in $\mR^{n \times m}$, $\bS$ and $\bP$ be non-singular covariance matrices. The product of densities can be refactorized as~\cite{AinKehStr02}
% \begin{equation}
%     \cN(\bx; \mu, \bP)~\cN(\by; \bA \bx, \bS) = \cN(\bx; \lambda, \Sigma)~\cN(\by; \omega, \Omega),
% \end{equation}
% where
% \begin{align*}
%     \lambda &= (\bI - \bH \bA) \mu + \bH \by \\
%     \Sigma &= (\bI - \bH \bA) \bP \\
%     \omega &= \bA \mu \\
%     \Omega &= \bS + \bA \bP \bA^T \\
%     \bH &= \bP \bA^T \Omega^{-1}.
% \end{align*}

%=========================================================
\subsection{Collective GHMM Inference: Collective Gaussian Forward-Backward Algorithm}
\label{subsec:collective_GHMM}

We have a total of $M$ trajectories of continuous observations over a single GHMM characterized by \eqref{eq:hmm_density}. The recorded observations are $\{ o_1^{(m)}, o_2^{(m)}, \ldots , o_t^{(m)}\},~\forall m=1,2,\ldots,M$ with $o_t^{(m)}$ being the continuous observations of the $m^{th}$ trajectory at time $t$. The objective of the collective GHMM inference is to estimate the distributions $n_t(x_t),~\forall t$. We assume that the aggregate observation is approximated by a Gaussian density at each time step $t$, i.e., 
\begin{equation}
    y_t(o_t) \sim \cN(o_t; \hat{\mu}_t, \hat{P}_t),
\end{equation}
where $\hat{\mu}_t$ and $\hat{P}_t$ can be estimated from the recorded observations.

For continuous states and continuous observations, the four messages in the CFB algorithm take the form

\begin{subequations}\label{eq:forward_backward_continuous}
\begin{eqnarray}
    \alpha_t(x_t) \propto \int  p(x_t|x_{t-1}) \alpha_{t-1} (x_{t-1}) \gamma_{t-1}(x_{t-1}) ~dx_{t-1}, \label{eq:forward_backwardcts1}  \\
    \beta_t(x_t) \propto  \int p(x_{t+1}|x_{t}) \beta_{t+1} (x_{t+1}) \gamma_{t+1}(x_{t+1})~dx_{t+1}, \label{eq:forward_backwardcts2} \\
    \gamma_t(x_t) \propto \int p(o_{t}|x_{t}) \frac{y_t(o_t)}{\xi_t(o_t)}~do_{t}, \label{eq:forward_backwardcts3} \\
    \xi_t(o_t) \propto \int p(o_{t}|x_{t}) \alpha_{t} (x_{t}) \beta_{t}(x_{t})~dx_{t}, \label{eq:forward_backwardcts4}
\end{eqnarray}
\end{subequations}
with boundary conditions 
\begin{equation*}
    \alpha_1(x_1) = p(x_1) \quad \text{and} \quad \beta_T(x_T) = 1.
\end{equation*}

Based on above, the messages in collective GHMMs are characterized by Theorem~\ref{thm:message_GHMM}. 

%====================================================
\begin{thm} \label{thm:message_GHMM}
The forward, backward, upward, and downward messages in collective GHMM take the following form, respectively.

\begin{subequations}
\begin{eqnarray}
    \alpha_{t}(x_{t})  \propto ~\mathrm{exp} \left( - \frac{1}{2} x^T \Lambda_t^{(f)} x + x^T \eta_t^{(f)}  \right) , \\
    \beta_{t}(x_{t}) \propto ~\mathrm{exp} \left( - \frac{1}{2} x^T \Lambda_t^{(b)} x + x^T \eta_t^{(b)}  \right), \\
     \gamma_{t}(x_{t}) \propto ~\mathrm{exp} \left( - \frac{1}{2} x^T \Lambda_t^{(u)} x + x^T \eta_t^{(u)}  \right), \\
      \xi_{t}(o_{t}) \propto ~\mathrm{exp} \left( - \frac{1}{2} x^T \Lambda_t^{(d)} x + x^T \eta_t^{(d)}  \right),
\end{eqnarray}
\end{subequations}
for $t=1,2,\ldots,T$. Here, the message parameters are the fixed points of the following recursive updates
% \begin{subequations} \label{eq:collective_ghmm_messages}
\begin{eqnarray*}
    \Lambda_t^{(f)} = Q^{-1} - Q^{-1} A( A^TQ^{-1}A + \Lambda_{t-1}^{(f)} + \Lambda_{t-1}^{(u)})^{-1} A^T Q^{-1} \\
    \eta_t^{(f)} = Q^{-1} A(A^T Q^{-1} A + \Lambda_{t-1}^{(f)} + \Lambda_{t-1}^{(u)})^{-1} (\eta_{t-1}^{(f)} + \eta_{t-1}^{(u)}) \\
    \Lambda_t^{(b)} = A^TQ^{-1} (Q^{-1} + \Lambda_{t+1}^{(b)} + \Lambda_{t+1}^{(u)})^{-1} (\Lambda_{t+1}^{(b)} + \Lambda_{t+1}^{(u)}) A\\
    \eta_t^{(b)} =  A^T Q^{-1} (Q^{-1} + \Lambda_{t+1}^{(b)} + \Lambda_{t+1}^{(u)})^{-1} (\eta_{t+1}^{(b)} + \eta_{t+1}^{(u)})\\
    \Lambda_t^{(d)} = R^{-1} - R^{-1} C( C^T R^{-1}C + \Lambda_{t}^{(f)} + \Lambda_{t}^{(b)})^{-1} C^T R^{-1} \\
    \eta_t^{(d)} =  R^{-1} C(C^T R^{-1} C + \Lambda_{t}^{(f)} + \Lambda_{t}^{(b)})^{-1} (\eta_{t}^{(f)} + \eta_{t}^{(b)}) \\
    \Lambda_t^{(u)} =  C^T (R + (\hat{P}_t^{-1} - \Lambda_{t}^{(d)})^{-1} )^{-1} C \\
    \eta_t^{(u)} =  C^T R^{-1} (R^{-1} + \hat{P}_t^{-1} - \Lambda_{t}^{(d)})^{-1} (\hat{P}_t^{-1} \hat{\mu}_{t} - \eta_{t}^{(d)})
\end{eqnarray*}
% \end{subequations}
with boundary conditions
\begin{align*}
    \Lambda_1^{(f)} = \Pi^{-1},~\eta_1^{(f)} = \Pi^{-1} \pi,\quad \Lambda_T^{(b)} = 0,~\eta_T^{(b)} = 0.
\end{align*}
Moreover, the required marginals can be computed as 
\begin{align*}
    n_t(x_t) &\propto \alpha_t(x_t)\beta_t(x_t)\gamma_t(x_t) \\
    &\propto \cN(x_t;\mu_t,P_t),
\end{align*}
where
\begin{subequations} \label{eq:ghmm_state_density_params}
\begin{eqnarray}
    P_t =  ( \Lambda_t^{(f)} +  \Lambda_t^{(b)} +  \Lambda_t^{(u)})^{-1} \\
    \mu_t = P_t (\eta_t^{(f)} + \eta_t^{(b)} + \eta_t^{(u)}).
\end{eqnarray}
\end{subequations}
%-----------------------------------------------
\end{thm}
\begin{proof}
See Appendix~\ref{appendix:proof_message_GHMM}
\end{proof}
%-----------------

%----------------------------
\begin{figure}[t]
\centering
\includegraphics[scale=0.8]{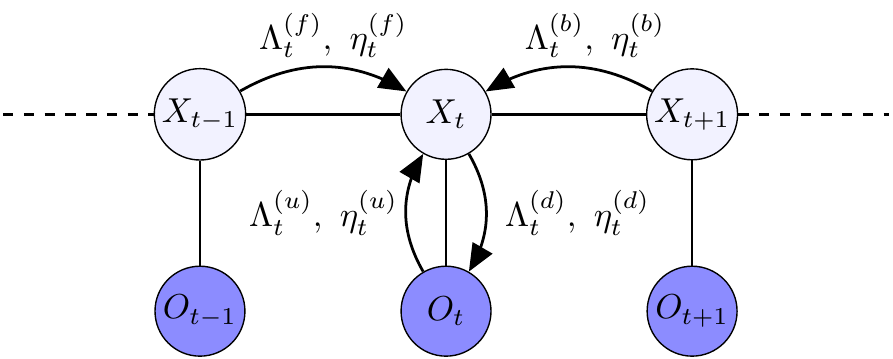}
\caption{Messages for inference in collective GHMM.}
\label{fig:ghmm_message}
\end{figure}
%------------------------------

%-----------------------------------------------------------
\begin{algorithm}[t]
   \caption{Collective Gaussian Forward-Backward (CGFB) Algorithm}
   \label{alg:collective_GHMM_forward_backward}
\begin{algorithmic}
   \STATE Initialize all the message parameters 
   \WHILE{not converged}
   \STATE \textbf{Forward pass:}
   \FOR{$t = 2,3,\ldots,T$}
        \STATE i) Update upward parameters $\Lambda_{t-1}^{(u)}$ and $\eta_{t-1}^{(u)}$
        \STATE ii) Update forward parameters $\Lambda_t^{(f)}$ and $ \eta_t^{(f)}$
        \STATE ii) Update downward parameters $\Lambda_t^{(d)}$ and $ \eta_t^{(d)}$
    \ENDFOR
    \STATE \textbf{Backward pass:}
    \FOR{$t = T-1,\ldots,1$}
        \STATE i) Update upward parameters $\Lambda_{t+1}^{(u)}$ and $\eta_{t+1}^{(u)}$
        \STATE ii) Update backward parameters $\Lambda_t^{(b)}$ and $ \eta_t^{(b)}$
        \STATE ii) Update downward parameters $\Lambda_t^{(d)}$ and $ \eta_t^{(d)}$
    \ENDFOR
    \ENDWHILE
    \STATE Estimate required state density parameters $\mu_t$ and $P_t$
\end{algorithmic}
\end{algorithm}
%-----------------------------------------------------

Based on the above Theorem, we propose collective Gaussian forward-backward (CGFB) algorithm (Algorithm~\ref{alg:collective_GHMM_forward_backward}) for aggregate inference in collective GHMMs. All the four kinds of messages involved in the algorithm are depicted in Figure~\ref{fig:ghmm_message}. The CGFB algorithm consists of two parts: a forward pass and a backward pass that are updated alternatively until convergence. Upon convergence, the required parameters of the aggregate state distributions can be computed using Equation~\eqref{eq:ghmm_state_density_params}.

%%%%%======================================================
\subsection{Sliding Window Filter for Fast Inference}
\label{subsec:sliding_window}

%Same as absorbing the prior in the node potential.
The size of the underlying GHMM increases with time and filtering becomes more and more expensive computationally. Therefore for real-time aggregate filtering, it is not suitable to consider the full length model. Similar to the recently proposed sliding window Sinkhorn belief propagation (SW-SBP) algorithm~\cite{SinHaaZha20_incremental}, the inference in collective GHMM can be performed in an incremental (online) fashion. Based on SW-SBP, we propose sliding window collective Gaussian forward backward (SW-CGFB) algorithm that employs a sliding window filter of length $K$ such that at each time step, only the most recent $K$ (aggregate) observations are taken into account for estimating the current state density. Moreover in order to encode the discarded information, the forward messages are utilized as the prior.
%----------------------------
\begin{figure}[h]
\centering
\includegraphics[scale=0.6]{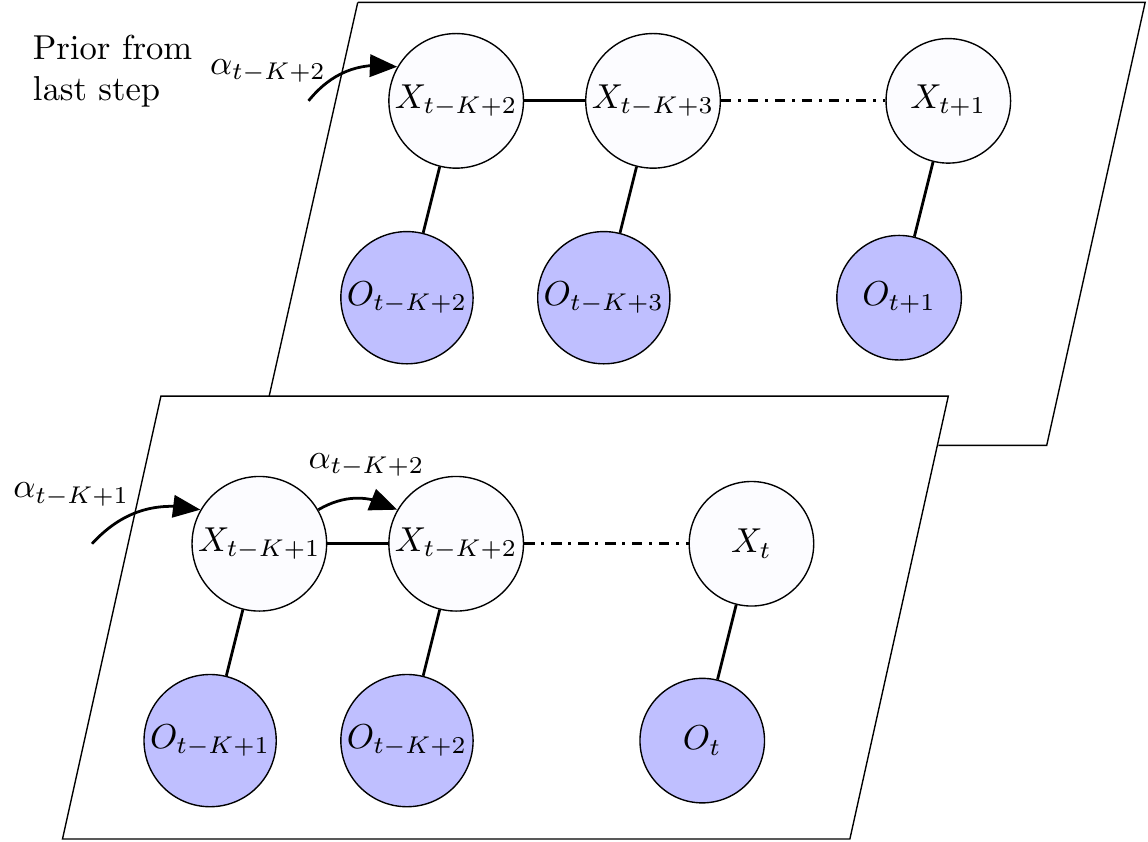}
\caption{Sliding window filter for incremental inference in collective GHMMs.}
\label{fig:sliding_window_ghmm}
\end{figure}
%------------------------------

The sliding window filter scheme for inference in collective GHMMs is illustrated in Figure~\ref{fig:sliding_window_ghmm} with window size of $K$. When the new observation comes in at time $t+1$, the left-most state variable is assigned a prior same as the forward message to the variable at time $t$ which is indicative of the discarded information. The CGFB algorithm is run in the $K$ length GHMM  with this prior represented by the forward message. 

%%%%%%%%%%%%%%%%%%%%%%%%%%%%%%%%%%%%%%%%%%%%%%%%%%%%%%%%%%%%
\section{Connections to Standard Kalman Filter}
\label{sec:connection_kalman}
% Given the observation sequence $\hat{o}_{1:T} = \{\hat{o}_1,\hat{o}_2,\ldots,\hat{o}_T\}$, the required marginal densities are~\cite{SinHaaZha20}
% \begin{align}
%     p(x_t|\hat{o}_{1:T}) &\propto p(\hat{o}_{1:T} | x_t) p(x_t) \nonumber \\
%     &= p(\hat{o}_{t} | x_t) \alpha_t(x_t) \beta_t(x_t),
% \end{align}
% where $\alpha_t(x_t) = p(x_t,\hat{o}_{1:t-1})$ denote the forward densities and $\beta_t(x_t) = p(\hat{o}_{t+1:T} | x_t)$ represent the backward densities. These densities take the form
% \begin{align}\label{eq:forward_standard}
%     \alpha_t(x_t) &= p(x_t,\hat{o}_{1:t-1}) \nonumber \\ 
%     &= \int p(x_t|x_{t-1}) \alpha_{t-1}(x_{t-1}) p(\hat{o}_{t-1}|x_{t-1})~dx_{t-1},
% \end{align}
% \begin{align}\label{eq:backward_standard}
%     \beta_t(x_t) &= p(\hat{o}_{t+1:T} | x_t) \nonumber \\ 
%     &= \int p(x_{t+1}|x_t)  \beta_{t+1}(x_{t+1}) p(\hat{o}_{t+1}|x_{t+1})~dx_{t+1}.
% \end{align}
% \chen{1. Kalman filter is standard. We can present the algorithm, but the derivation is not needed. 2. In case $M=1$, the {\bf sliding window CGFB with window size $K=1$} reduces to Kalman filter.}
Suppose we know the estimated state distribution $n_t(x_t) \propto \cN(x_t; \mu_t, P_t)$ and as the single new observation $\hat{o}_{t+1}$ is recorded at time $t+1$, the goal is to estimate the state distribution $n_{t+1}(x_{t+1}) = p(x_{t+1}|\hat{o}_{1:t+1})$. This is the setting of the standard Kalman filter. We show that the standard Kalman filter is equivalent to the SW-CGFB algorithm for the case of a single individual ($M=1$) and unit window size ($K=1$).

The sliding window filter scheme with window size $K=1$ for a single individual is shown in Figure~\ref{fig:kalman_message}. Note that here the backward messages do not contribute to the state density estimation due to the boundary condition in Algorithm~\ref{alg:collective_GHMM_forward_backward}. Moreover, the upward messages do not depend on the downward message and reduce to the corresponding observation probabilities of the underlying GHMM~\cite{SinHaaZha20}. 

%----------------------------
\begin{figure}[h]
\centering
\includegraphics[scale=0.8]{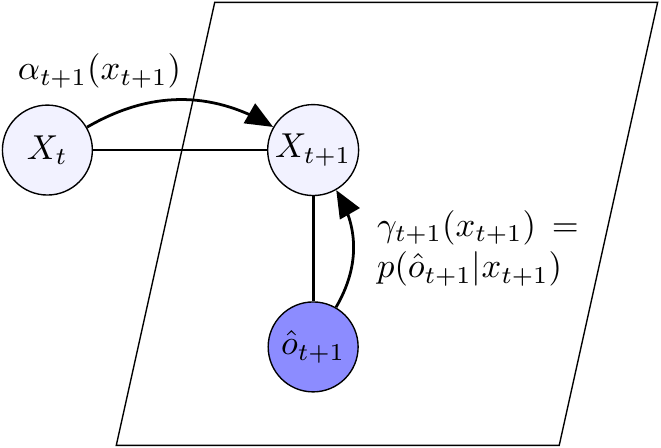}
\caption{Sliding window filter with window size $K=1$ and number of agents $M=1$. }
\label{fig:kalman_message}
\end{figure}
%------------------------------

Identifying the forward message $\alpha_{t+1}(x_{t+1})$ as prediction step and then incorporating the upward message from the observation node as correction step, the state density can be estimated as 
\begin{equation}
   n_{t+1}(x_{t+1}) \propto ~ \underbrace{\underbrace{\alpha_{t+1}(x_{t+1})}_{\mathbf{prediction}} ~~\gamma_{t+1}(x_{t+1}) }_{\mathbf{correction}}.
\end{equation}

Here the predicted state density takes the form
\begin{align*}
    \alpha_{t+1} (x_{t+1}) \propto&~ \int p(x_{t+1} | x_t) \alpha_t(x_t) \gamma_t(x_t) ~dx_t \\
    \propto&~ \int p(x_{t+1} | x_t) ~n_t(x_t)~dx_t \\
    =&~ \int \cN(x_{t+1}; Ax_t, Q)~ \cN(x_t; \mu_t, P_t)  ~dx_t \\
    \propto&~  \cN(x_{t+1}; \mu_{t+1 | t}, P_{t+1|t}) ,
\end{align*}
with 
\begin{subequations}\label{eq:kalman_prediction}
\begin{eqnarray}
    P_{t+1|t} = Q + A P_{t|t} A^T \\
    \mu_{t+1|t} = A \mu_{t|t}.
\end{eqnarray}
\end{subequations}
Note that in general, the parameters $\mu_t$ and $P_t$ of the estimated aggregate state distribution $n_t(x_t)$ also depend on the future observation. However, in the setting of sliding window filter, they are same as $\mu_{t|t}$ and $P_{t|t}$ taking only the current and past observations into account. After the prediction step, the correction step aims to estimate the required state distribution as
\begin{align*}
    n_{t+1}(x_{t+1}) \propto& ~ \alpha_{t+1}(x_{t+1}) \gamma_{t+1}(x_{t+1}) \\
    \propto& ~ \cN(x_{t+1}; \mu_{t+1 | t}, P_{t+1|t})~ p(\hat{o}_{t+1} | x_{t+1})\\
    \propto& ~ \cN(x_{t+1}; \mu_{t+1 | t}, P_{t+1|t}) ~\cN(\hat{o}_{t+1}; C x_{t+1}, R)  \\
    \propto& ~ \mathrm{exp}\left( - \frac{1}{2} x_{t+1}^T (P_{t+1|t}^{-1} + C^T R^{-1} C) x_{t+1} \right. \\
    & \quad \quad \quad \quad \left. +~ x^T (P_{t+1|t}^{-1}\mu_{t+1 | t} + C^T R^{-1} \hat{o}_{t+1} ) \right) \\
    \propto& ~ \mathrm{exp}\left( - \frac{1}{2} (x_{t+1} - \mu_{t+1|t+1})^T (P_{t+1|t}^{-1}  \right. \\ 
    & \quad \quad \quad \quad \left. + C^T R^{-1} C) (x_{t+1} - \mu_{t+1|t+1}) \right) \\
    \propto& ~ \cN(x_{t+1}; \mu_{t+1|t+1}, P_{t+1|t+1}),
\end{align*}
with
\begin{subequations}\label{eq:kalman_correction}
\begin{eqnarray}
    P_{t+1|t+1} = (I - K_{t+1} C) P_{t+1|t} \\
    \mu_{t+1|t+1} = \mu_{t+1|t} + K_{t+1} (\hat{o}_{t+1} - C \mu_{t+1| t}) \\
    K_{t+1} = P_{t+1|t} C^T ( R + C P_{t+1|t} C^T)^{-1}.
\end{eqnarray}
\end{subequations}

Equations \eqref{eq:kalman_prediction} and \eqref{eq:kalman_correction} describe the prediction and correction updates of the standard Kalman filter, respectively, and its full derivation can be found in Appendix~\ref{appendix:kalman_proof1}. %Moreover, the standard Kalman filter updates can also be deduced from Algorithm~\ref{alg:collective_GHMM_forward_backward} as stated in the following proposition. 

% %---------------------------
% \begin{prop}\label{prop:kalman_relation}
% The standard Kalman filter updates given by \eqref{eq:kalman_prediction} and \eqref{eq:kalman_correction} are special case of \eqref{eq:collective_ghmm_messages} with a single individual (M=1).
% \end{prop}
% \begin{proof}
% See Appendix~\ref{appendix:kalman_relation_proof}
% \end{proof}
% %---------------------------------

\begin{rem}
Similarly, it can be shown that the message updates in Algorithm~\ref{alg:collective_GHMM_forward_backward} reduce to the standard Kalman smoother updates for the case of a single individual, i.e., $M=1$.
\end{rem}

%%%%%%%%%%%%%%%%%%%%%%%%%%%%%%%%%%%%%%%%%%%%%%%%%%
\section{Numerical Experiments}
\label{sec:experiments}
We conduct multiple experiments to evaluate the performance of our CGFB algorithm. We simulate a system with the GHMM parameters:
\begin{align*}
    \pi &= \begin{bmatrix} 1 \\ 0 \end{bmatrix}, \quad \Pi = \begin{bmatrix} 1 & 0.2 \\ 0.2 & 1 \end{bmatrix} \\
    A &= \begin{bmatrix} 1 & \Delta t \\ -\Delta t & 1 - 0.5 \Delta t \end{bmatrix}, \quad C = \begin{bmatrix} 0 & \Delta t\end{bmatrix}\\
    Q &= \Delta t \begin{bmatrix} 0.1 & 0 \\ 0 & 0.1 \end{bmatrix}, \quad R = \Delta t \begin{bmatrix} 0.7 \end{bmatrix},
\end{align*}
where $\Delta t$ is set to $0.05$ for all our experiments. Based on the above GHMM parameters, we generate $M$ number of trajectories and record the aggregate observations $\{ o_1^{(m)}, o_2^{(m)}, \ldots , o_T^{(m)}\},~\forall m=1,2,\ldots,M$. We then use CGFB algorithm to estimate the aggregate state distributions $n_t(x_t),~\forall t=1,2,\ldots,T$, which are compared against the ground truth distributions in terms of quadratic error averaged over all the time steps $t$. Figure~\ref{fig:convergence} depicts error curves for mean and covariance estimates with respect to the corresponding ground truth means and covariances. It can be clearly observed that the algorithm converges and estimated parameters are very close to the ground truth. 
%----------------------------
\begin{figure}[h]
\centering
\includegraphics[scale=0.18]{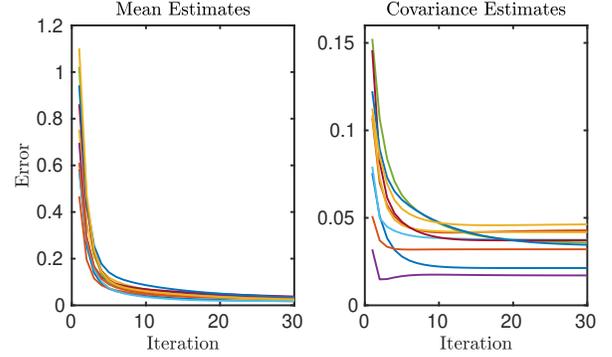}
\caption{Convergence of CGFB with respect to the corresponding ground truths with $M=200$ and $T=100$. Different colors represent different random seeds.}
\label{fig:convergence}
\end{figure}
%------------------------------
%------------------------------
\begin{figure}[h]
\centering
\includegraphics[scale=0.18]{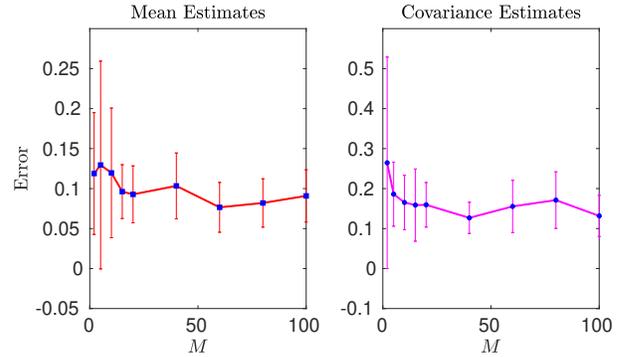}
\caption{Quadratic errors with respect to the ground truth with $T=100$ and varying number of agents $M$. The plot is averaged over 10 different seeds; the vertical bars represent the corresponding standard deviations.}
\label{fig:different_M}
\end{figure}
%------------------------------

Next, we study the effect of number of agents $M$ on the estimation of aggregate state distributions. Figure~\ref{fig:different_M} shows the error values for different number of agents. We observe that errors in both mean and covariance estimates get smaller as the number of agents $M$ increase.  Furthermore, we evaluate the performance of CGFB against $M$ independent standard Kalman filters assuming the identities of each particle and measurement are known. Let the Kalman filter mean and covariance estimates for $M$ agents at time $T$ be $\mu_T^{m}$  and $P_T^{m}$ for $m=1,2,\ldots, M$. We compute the mean and variances of these $M$ Kalman estimations as
\begin{align*}
    \mu_T^{\mathrm{KF}} &= \frac{1}{M} \sum_{m=1}^M \mu_T^m \\
    P_T^{\mathrm{KF}} &= \frac{1}{M} \sum_{m=1}^M P_T^m + (\mu_T^{m} - \mu_T^{\mathrm{KF}}) (\mu_T^{m} - \mu_T^{\mathrm{KF}})^T
\end{align*}
and use them to compare against the estimates obtained by the CGFB algorithm. This comparison in terms of quadratic error is shown in Figure~\ref{fig:kalman_compare}.

%----------------------------
\begin{figure}[h]
\centering
\includegraphics[scale=0.17]{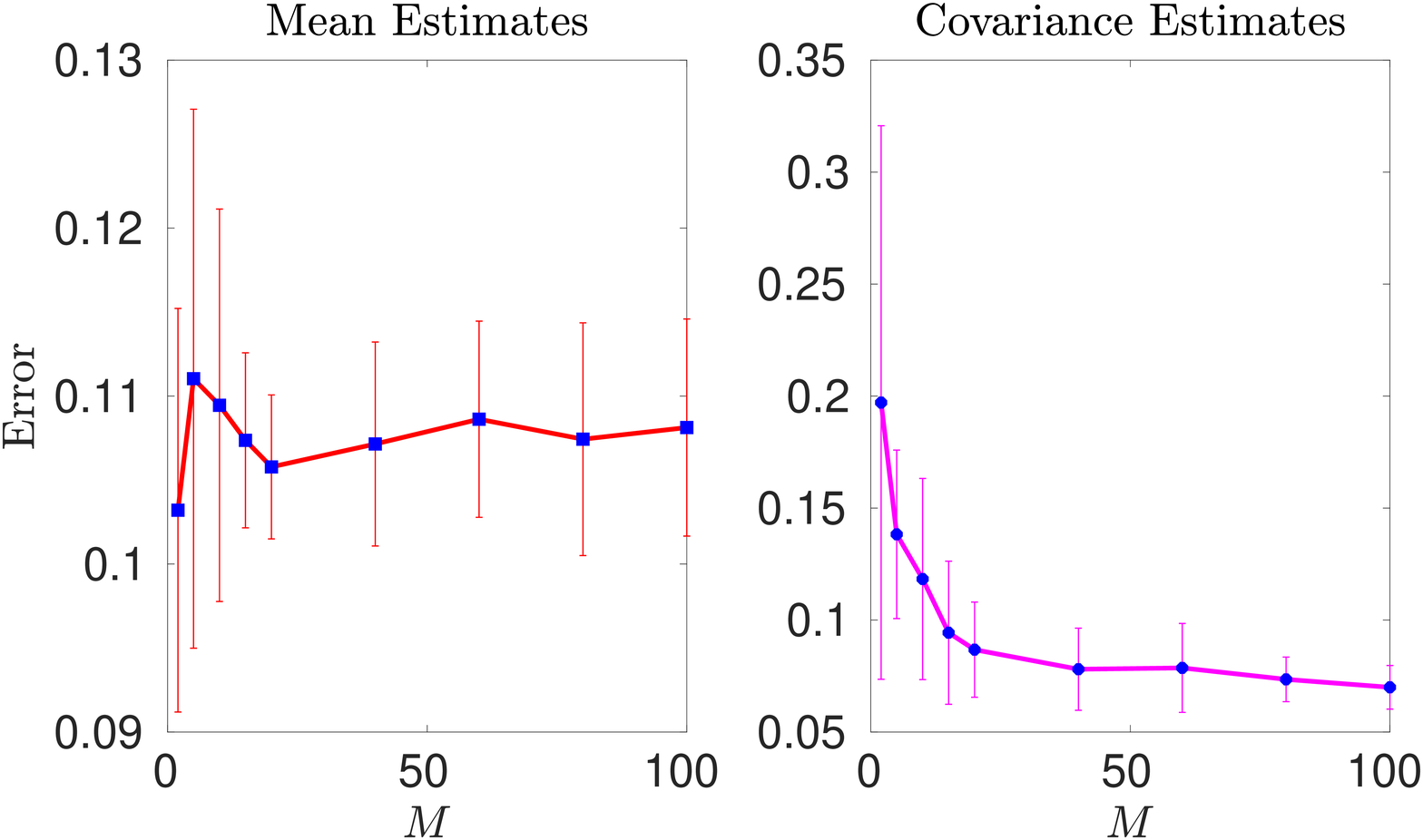}
\caption{Error comparison against Kalman filter estimates for $T=100$}
\label{fig:kalman_compare}
\end{figure}
%------------------------------

%----------------------------
\begin{figure}[h]
\centering
\includegraphics[scale=0.18]{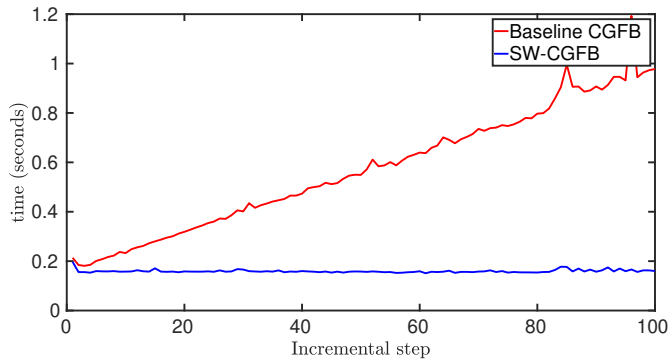}
\caption{Comparison of time taken by the baseline CGFB (full length graph) and SW-CGFB (sliding window). Here, $M=100,~K=20$. }
\label{fig:time_plot}
\end{figure}
%------------------------------
%----------------------------
\begin{figure}[h]
\centering
\includegraphics[scale=0.5]{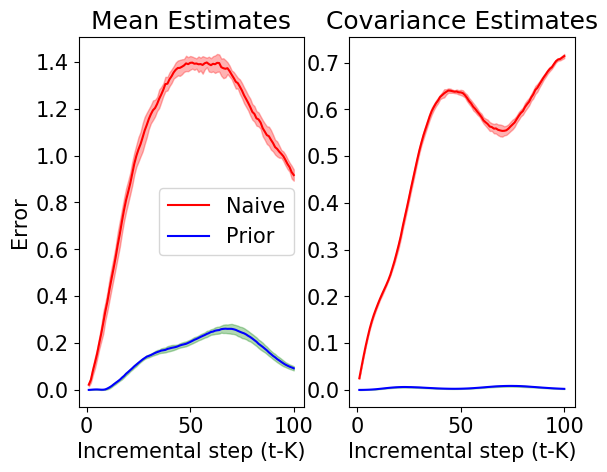}
\caption{Sliding window filter with window size $K=20$ and number of agents $M=100$. }
\label{fig:sliding_window_K20}
\end{figure}
%------------------------------

%----------------------------
\begin{figure}[h]
\centering
\includegraphics[scale=0.5]{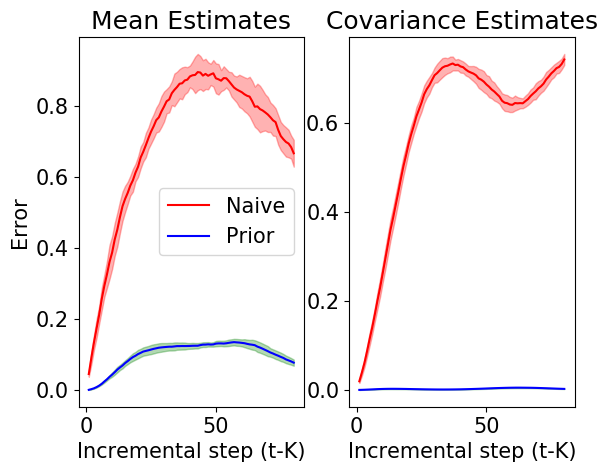}
\caption{Error comparison of sliding window filters with window size $K=30$ and number of agents $M=100$.}
\label{fig:sliding_window_K30}
\end{figure}
%------------------------------

Finally, we evaluate the performance of our algorithm with sliding window filter. First, we present the comparison of time complexity between the full length graph that we refer as baseline CGFB and SW-CGFB in Figure~\ref{fig:time_plot}. It is evident that the time taken by the baseline CGFB increases linearly with time, while SW-CGFB takes a constant amount of time at each step. Now we compare the SW-CGFB algorithm against the baseline. The naive version of SW-CGFB refers to the sliding window filter without considering the prior from the previous step. Figures~\ref{fig:sliding_window_K20} and \ref{fig:sliding_window_K30} depict the error performance for different values of window sizes. In both the figures, the results are averaged over 10 different seeds and the shaded areas represent the corresponding standard deviations. It can be observed that the naive sliding window filter without considering the prior performs poorly as compared to the one with prior. Moreover, the errors in the mean estimates decrease with the increase in window size.

%%%%%%%%%%%%%%%%%%%%%%%%%%%%%%%%%%%%%%%%%%%%%%%%%%
\section{Conclusion}
\label{sec:conclusion}
In this paper, we proposed an algorithm for filtering and inference in collective GHMMs, where the data from several individuals is recorded in aggregate form such that the observations are indistinguishable. The algorithm enjoys convergence guarantees and naturally reduces to the standard Kalman filter for the case of a single individual. We also proposed a sliding window filter scheme for fast inference in collective GHMMs.

\bibliography{gauss_hmm}
\bibliographystyle{IEEEtran}

%%%%%%%%%%%%%%%%%%%%%%%%%%%%%%%%%%%%%%%%%%

\appendix

\subsection{Proof of Theorem~\ref{thm:message_GHMM}}
\label{appendix:proof_message_GHMM}
\begin{proof}
The product messages take the form
\begin{align*}
     \alpha_t(x_t) \gamma_t(x_t) &\propto~ \mathrm{exp} \left(- \frac{1}{2} x^T (\Lambda_t^{(f)} + \Lambda_t^{(u)}) x + x^T (\eta_t^{(f)} + \eta_t^{(u)}) \right)\\
     \beta_t(x_t) \gamma_t(x_t) &\propto ~ \mathrm{exp} \left(- \frac{1}{2} x^T (\Lambda_t^{(b)} + \Lambda_t^{(u)}) x + x^T (\eta_t^{(b)} + \eta_t^{(u)}) \right)\\
      \alpha_t(x_t) \beta_t(x_t) &\propto~ \mathrm{exp} \left(- \frac{1}{2} x^T (\Lambda_t^{(f)} + \Lambda_t^{(b)}) x + x^T (\eta_t^{(f)} + \eta_t^{(b)}) \right).
\end{align*}

% where
% \begin{subequations}
% \begin{eqnarray}
%      P_t^{(fu)}=((P_t^{(f)})^{-1} + \Lambda_t^{(u)})^{-1}, \quad  \mu_t^{(fu)}=(P_t^{(fu)})^{-1} ((P_t^{(f)})^{-1} \mu_t^{(f)} + \eta_t^{(u)}) \\
%      P_t^{(bu)}= ((P_t^{(b)})^{-1} + \Lambda_t^{(u)})^{-1} , \quad \mu_t^{(bu)}=(P_t^{(bu)})^{-1} ((P_t^{(b)})^{-1} \mu_t^{(b)} + \eta_t^{(u)}) \\
%      P_t^{(fb)}= ((P_t^{(f)})^{-1} + (P_t^{(b)})^{-1} )^{-1}, \quad \mu_t^{(fb)} = (P_t^{(fb)})^{-1} ((P_t^{(f)})^{-1} \mu_t^{(f)} + (P_t^{(b)})^{-1} \mu_t^{(b)}).
% \end{eqnarray}
% \end{subequations}

\textbf{Forward Messages:}

The forward messages can be written as
\begin{align*}
    \alpha_t(x_t) &\propto \int p(x_t|x_{t-1}) \alpha_{t-1} (x_{t-1}) \gamma_{t-1}(x_{t-1}) ~dx_{t-1} \\
    &\propto \int \mathrm{exp} \left(- \frac{1}{2} (x_t - A x_{t-1})^T Q^{-1} (x_t - A x_{t-1})  \right) \\
    & \quad \mathrm{exp} \left(- \frac{1}{2} x_{t-1}^T (\Lambda_{t-1}^{(f)} + \Lambda_{t-1}^{(u)})x_{t-1}  + x_{t-1}^T ( \eta_{t-1}^{(f)} \right. \\
    & \quad \left. + \eta_{t-1}^{(u)}) \right) ~dx_{t-1} \\
    &\propto \cN \left( x_t; A \mu_{t-1}^{(fu)}, Q + A P_{t-1}^{(fu)} A^T \right),
\end{align*}

where $\mu_{t-1}^{(fu)}$ and $P_{t-1}^{(fu)}$ are the mean and covariance of the product message $\alpha_{t-1}(x_{t-1}) \gamma_{t-1}(x_{t-1})$, respectively. The above equation simplifies in canonical form as 

\begin{align*}
    \Lambda_t^{(f)} &= (Q + A P_{t-1}^{(fu)} A^T)^{-1} \\
    &= Q^{-1} - Q^{-1} A( A^TQ^{-1}A + \Lambda_{t-1}^{(fu)})^{-1} A^T Q^{-1} \\
    &= Q^{-1} - Q^{-1} A( A^TQ^{-1}A + \Lambda_{t-1}^{(f)} + \Lambda_{t-1}^{(u)})^{-1} A^T Q^{-1}
\end{align*}

\begin{align*}
    \eta_t^{(f)} &= \Lambda_t^{(f)} A \mu_{t-1}^{(fu)} \\
    &= (Q + A P_{t-1}^{(fu)} A^T )^{-1} (A P_{t-1}^{(fu)} \eta_{t-1}^{(fu)}) \\
    &= Q^{-1} A(A^T Q^{-1} A + \Lambda_{t-1}^{(fu)})^{-1} \eta_{t-1}^{(fu)} \\
    &= Q^{-1} A(A^T Q^{-1} A + \Lambda_{t-1}^{(f)} + \Lambda_{t-1}^{(u)})^{-1} (\eta_{t-1}^{(f)} + \eta_{t-1}^{(u)}).
\end{align*}

% Using Proposition~\ref{prop:gaussian_variable_change}, and Proposition~\ref{prop:gaussian_marginalization}, the backward density takes the form
% \begin{align*}
%     \beta_t(x_t) &\propto  \int p(x_{t+1}|x_{t}) \beta_{t+1} (x_{t+1}) \gamma_{t+1}(x_{t+1})~dx_{t+1} 
% \end{align*}

\textbf{Backward Messages:}

\begin{align*}
    p(x_{t+1}|x_{t}) & \beta_{t+1} (x_{t+1}) \gamma_{t+1}(x_{t+1})\\
    \propto &~ \mathrm{exp} \left( -\frac{1}{2} (x_{t+1} - Ax_t)^T Q^{-1} (x_{t+1} -Ax_t) \right. \\
    & \left. -\frac{1}{2} (x_{t+1} - \mu_{t+1}^{(b)})^T (\Lambda_{t+1}^{(b)}) (x_{t+1} - \mu_{t+1}^{(b)}) \right. \\
    & \left.  -\frac{1}{2} (x_{t+1} - \mu_{t+1}^{(u)})^T (\Lambda_{t+1}^{(u)}) (x_{t+1} - \mu_{t+1}^{(u)}) \right) \\
    \propto&~ \mathrm{exp} \left( -\frac{1}{2} (x_{t+1} - \nu)^T \Sigma_1 (x_{t+1} - \nu) \right. \\
    & \left. -\frac{1}{2} (Ax_{t} - \mu_{t+1}^{(b)})^T \Sigma_2 (A x_{t} - \mu_{t+1}^{(b)}) \right. \\
    & \left. -\frac{1}{2} (Ax_{t} - \mu_{t+1}^{(u)})^T \Sigma_3 (A x_{t} - \mu_{t+1}^{(u)}) \right),
\end{align*}
where
\begin{subequations}
\begin{eqnarray}
    \Sigma_1 = (Q^{-1} + \Lambda_{t+1}^{(b)} + \Lambda_{t+1}^{(u)})^{-1}  \\
    \Sigma_2 = Q^{-1} \Sigma_1 \Lambda_{t+1}^{(b)} \\
    \Sigma_3 = Q^{-1} \Sigma_1 \Lambda_{t+1}^{(u)} 
\end{eqnarray}
\end{subequations}
Therefore,
\begin{align*}
    \beta_t(x_t) &\propto~ \int p(x_{t+1}|x_{t}) \beta_{t+1} (x_{t+1}) \gamma_{t+1}(x_{t+1}) ~dx_{t+1} \\
    &\propto~\mathrm{exp} \left( - \frac{1}{2}\{ x_{t}^T A^T (\Sigma_2 + \Sigma_3)A x \right. \\
    & \quad \quad \quad \quad \left. +~ x_t^T (A^T \Sigma_2 \mu_{t+1}^{(b)} + A^T \Sigma_3 \mu_{t+1}^{(u)}) \}   \right) \\
    & \propto~\mathrm{exp} \left( - \frac{1}{2}\{ x_{t}^T \Lambda_t^{(b)} x + x_t^T \eta_t^{(b)} \}   \right)
\end{align*}
with
\begin{align*}
    \Lambda_t^{(b)} &=  A^T (\Sigma_2 + \Sigma_3)A \\
    &= A^TQ^{-1} (Q^{-1} + \Lambda_{t+1}^{(b)} + \Lambda_{t+1}^{(u)})^{-1} (\Lambda_{t+1}^{(b)} + \Lambda_{t+1}^{(u)}) A \\
    \eta_t^{(b)} &= A^T \Sigma_2 \mu_{t+1}^{(b)} + A^T \Sigma_3 \mu_{t+1}^{(u)} \\
    &= A^T Q^{-1} (Q^{-1} + \Lambda_{t+1}^{(b)} + \Lambda_{t+1}^{(u)})^{-1} (\Lambda_{t+1}^{(b)} \mu_{t+1}^{(b)} + \Lambda_{t+1}^{(u)} \mu_{t+1}^{(u)}) \\
    &= A^T Q^{-1} (Q^{-1} + \Lambda_{t+1}^{(b)} + \Lambda_{t+1}^{(u)})^{-1} (\eta_{t+1}^{(b)} + \eta_{t+1}^{(u)})
\end{align*}

%===================================
\textbf{Downward Messages:}

The downward density can be simplified as 
\begin{align*}
    \xi_t(o_t) &\propto \int p(o_{t}|x_{t}) \alpha_{t} (x_{t}) \beta_{t}(x_{t})~dx_{t} \\
    &\propto \int \cN(o_t; Cx_t,R) \cN\left(x_t; \mu_{t}^{(fb)}, P_{t}^{(fb)}\right) ~dx_{t} \\
    &\propto \cN\left( o_t; C \mu_t^{(fb)}, R + C P_t^{(fb)} C^T \right),
\end{align*}
where $\mu_{t}^{(fb)}$ and $P_{t}^{(fb)}$ are the mean and covariance of product message $\alpha_{t}(x_{t}) \beta_{t}(x_{t})$, respectively. In canonical form, it further reduces to

\begin{align*}
    \Lambda_t^{(d)} &= (R + C P_{t}^{(fb)} C^T)^{-1} \\
    &= R^{-1} - R^{-1} C( C^T R^{-1}C + \Lambda_{t}^{(fb)})^{-1} C^T R^{-1} \\
    &= R^{-1} - R^{-1} C( C^TR^{-1}C + \Lambda_{t}^{(f)} + \Lambda_{t}^{(b)})^{-1} C^T R^{-1}
\end{align*}

\begin{align*}
    \eta_t^{(d)} &= \Lambda_t^{(d)} C \mu_{t}^{(fb)} \\
    &= (R + C P_{t}^{(fb)} C^T )^{-1} (C P_{t}^{(fb)} \eta_{t}^{(fb)}) \\
    &= R^{-1} C(C^T R^{-1} C + \Lambda_{t}^{(fb)})^{-1} \eta_{t}^{(fb)} \\
    &= R^{-1} C(C^T R^{-1} C + \Lambda_{t}^{(f)} + \Lambda_{t}^{(b)})^{-1} (\eta_{t}^{(f)} + \eta_{t}^{(b)}).
\end{align*}
%============================================
\textbf{Upward Messages:}\\

Assuming that the output samples take the form of Gaussian density
\begin{equation}
    y_t(o_t) = \cN(o_t; \hat{\mu}_t, \hat{P}_t),
\end{equation}

\begin{align*}
    p(o_{t}|x_{t}) \frac{y_t(o_t)}{\xi_t(o_t)}  \propto &~ \mathrm{exp} \left( -\frac{1}{2} o_{t} - Cx_t)^T R^{-1} (o_{t} - Cx_t) \right. \\
    & \left. -\frac{1}{2} (o_{t} - \hat{\mu}_{t})^T \hat{P}_t^{-1} (o_{t} - \hat{\mu}_t) \right. \\
    & \left.  +\frac{1}{2} (o_{t} - \mu_{t}^{(d)})^T \Lambda_{t}^{(d)} (o_{t} - \mu_{t}^{(d)}) \right) \\
    \propto&~ \mathrm{exp} \left( -\frac{1}{2} (o_{t} - \nu)^T \Sigma_1 (o_{t} - \nu) \right. \\
    & \left. - \frac{1}{2} (C x_{t} - \hat{\mu}_{t})^T \Sigma_2 (C x_{t} - \hat{\mu}_{t}) \right. \\
    & \left. - \frac{1}{2} (Cx_{t} - \mu_{t}^{(d)})^T \Sigma_3 (C x_{t} - \mu_{t}^{(d)}) \right),
\end{align*}
where
\begin{subequations}
\begin{eqnarray}
    \Sigma_1 = (R^{-1} + \hat{P}_t^{-1} - \Lambda_{t}^{(d)})^{-1}  \\
    \Sigma_2 = R^{-1} \Sigma_1 \hat{P}_t^{-1} \\
    \Sigma_3 = - R^{-1} \Sigma_1 \Lambda_{t}^{(d)} 
\end{eqnarray}
\end{subequations}
Therefore,
\begin{align*}
    \gamma_t(x_t) &\propto~ \int  p(o_{t}|x_{t}) \frac{y_t(o_t)}{\xi_t(o_t)} ~do_{t} \\
    &\propto~\mathrm{exp} \left( - \frac{1}{2}  x_{t}^T C^T (\Sigma_2 + \Sigma_3)C x  \right. \\
    & \quad \quad \quad \quad \left. +~ x_t^T (C^T \Sigma_2 \hat{\mu}_t + C^T \Sigma_3 \mu_{t}^{(u)})    \right) \\
    & \propto~\mathrm{exp} \left( - \frac{1}{2} x_{t}^T \Lambda_t^{(u)} x + x_t^T \eta_t^{(u)}    \right)
\end{align*}
with
\begin{align*}
    \Lambda_t^{(u)} &=  C^T (\Sigma_2 + \Sigma_3)C \\
    &= C^T R^{-1} (R^{-1} + \hat{P}_t^{-1} - \Lambda_{t}^{(d)})^{-1}  (\hat{P}_t^{-1} - \Lambda_{t}^{(d)}) C \\
    &= C^T (R + (\hat{P}_t^{-1} - \Lambda_{t}^{(d)})^{-1} )^{-1} C \\
    \eta_t^{(u)} &= C^T \Sigma_2 \hat{\mu}_t + C^T \Sigma_3 \mu_{t}^{(u)} \\
    &= C^T R^{-1}  (R^{-1} + \hat{P}_t^{-1} - \Lambda_{t}^{(d)})^{-1}  (\hat{P}_t^{-1} \hat{\mu}_{t} - \Lambda_{t}^{(d)} \mu_{t}^{(d)}) \\
    &= C^T R^{-1} (R^{-1} + \hat{P}_t^{-1} - \Lambda_{t}^{(d)})^{-1} (\hat{P}_t^{-1} \hat{\mu}_{t} - \eta_{t}^{(d)})
\end{align*}
\end{proof}

%=================================================================
\subsection{Proof of Kalman updates \eqref{eq:kalman_prediction} and \eqref{eq:kalman_correction}}
\label{appendix:kalman_proof1}

\begin{proof}
\begin{align*}
    P_{t+1|t+1} &= (P_{t+1|t}^{-1} + C^T R^{-1} C)^{-1} \\
    &= P_{t+1|t} - P_{t+1|t} C^T (R + C P_{t+1|t} C^T )^{-1} C P_{t+1|t} \\
    &= (I - P_{t+1|t} C^T (R + C P_{t+1|t} C^T )^{-1} C) P_{t+1|t} \\
    &= (I - K_{t+1} C) P_{t+1|t},
\end{align*}
where
\begin{align*}
    K_{t+1} =  P_{t+1|t} C^T (R + C P_{t+1|t} C^T )^{-1}.
\end{align*}
Moreover, 
\begin{align*}
    \mu_{t+1|t+1} &=  P_{t+1|t+1} (P_{t+1|t}^{-1}\mu_{t+1 | t} + C^T R^{-1} \hat{o}_{t+1} )   \\
    &= (I - K_{t+1} C) P_{t+1|t} (P_{t+1|t}^{-1}\mu_{t+1 | t} + C^T R^{-1} \hat{o}_{t+1} )   \\
    &= (I - K_{t+1} C) \mu_{t+1 | t} + (I - K_{t+1} C) P_{t+1|t} C^T R^{-1} \hat{o}_{t+1} \\
    &= \mu_{t+1 | t} - K_{t+1} C\mu_{t+1 | t} + P_{t+1|t} C^T R^{-1} \\
    & \quad \quad \hat{o}_{t+1} - K_{t+1} C P_{t+1|t} C^T R^{-1} \hat{o}_{t+1} \\
    &= \mu_{t+1 | t} - K_{t+1} C\mu_{t+1 | t} \\
    & \quad \quad + (P_{t+1|t} C^T R^{-1}  - K_{t+1} C P_{t+1|t} C^T R^{-1}) \hat{o}_{t+1} \\
    &= \mu_{t+1 | t} - K_{t+1} C\mu_{t+1 | t} + (K_{t+1} (R + C P_{t+1|t} C^T ) \\
    & \quad \quad R^{-1}  - K_{t+1} C P_{t+1|t} C^T R^{-1}) \hat{o}_{t+1} \\
    &=  \mu_{t+1 | t} - K_{t+1} C\mu_{t+1 | t} + (K_{t+1}) \hat{o}_{t+1} \\
    &= \mu_{t+1 | t} + K_{t+1}( \hat{o}_{t+1} - C\mu_{t+1 | t} ).
\end{align*}
\end{proof}

\end{document}